\newif\ifarxiv
\arxivtrue
\ifarxiv
\documentclass[letter,11pt]{article}
\else
\documentclass[final,12pt]{alt2021}
\fi

\ifarxiv
\title{Efficient Learning with Arbitrary Covariate Shift}
\else
\title[Efficient Learning with Arbitrary Covariate Shift]{Efficient Learning with Arbitrary Covariate Shift}
\fi
\usepackage{times}

\ifarxiv
\author{Adam Kalai \\ \texttt{adum@microsoft.com} \\ Microsoft Research New England \and Varun Kanade \\ \texttt{varunk@cs.ox.ac.uk} \\ University of Oxford}
\else
\altauthor{%
 \Name{Adam Kalai} \Email{adum@microsoft.com}\\
 \addr Microsoft Research New England
 \AND
 \Name{Varun Kanade} \Email{varunk@cs.ox.ac.uk}\\
 \addr University of Oxford%
}
\fi

\usepackage[utf8]{inputenc}
\usepackage{mathtools}
\usepackage{enumitem}
\usepackage{tikz}
\usetikzlibrary{patterns}

\ifarxiv
\usepackage{subcaption}
\usepackage[ruled,vlined]{algorithm2e}
\usepackage{amssymb,amsthm,amsmath}
\usepackage{graphicx}
\usepackage{url}
\usepackage{fullpage}
\usepackage{xspace}
\usepackage{hyperref}
\hypersetup{
    colorlinks,
    linkcolor={red!50!black},
    citecolor={blue!50!black},
    urlcolor={blue!80!black}
}
\usepackage[round]{natbib}
\newtheorem{definition}{Definition}
\newtheorem{lemma}{Lemma}
\newtheorem{theorem}{Theorem}

\fi

\newcommand{\poly}{\mathrm{poly}}
\newcommand{\opt}{\mathrm{opt}}
\newcommand{\indicator}{\mathbf{1}}
\newcommand{\EX}{\mathsf{EX}}
\newcommand{\DNF}{\mathsf{DNF}}
\newcommand{\err}{\mathrm{err}}
\newcommand{\false}{\mathrm{false}}
\newcommand{\rej}{\mathrm{rej}}

\newcommand{\sgn}{\mathrm{sgn}}
\newcommand{\defeq}{\coloneqq}
\newcommand{\eps}{\epsilon}
\newcommand{\question}{\bot}
\newcommand{\zo}{\{0,1\}}

\newcommand{\VC}{\mathrm{VC}}
\newcommand{\PAR}{\mathsf{PARITIES}}
\newcommand{\GF}{\mathbf{GF}}
\newcommand{\vspan}{\mathrm{span}}
\newcommand{\algfullname}{{Slice-and-Dice}\xspace}
\newcommand{\alg}{S\&D\xspace}
\newcommand{\algneg}{DICE$_0$\xspace}
\newcommand{\algpos}{DICE$_1$\xspace}

\makeatletter
\renewcommand{\P}{%
	\@ifnextchar\bgroup%
	{\@Pwithargs}
	{\@Pnoargs}
}
\newcommand{\@Pwithargs}[1]{%
	\@ifnextchar\bgroup%
	{\@Ptwoargs{#1}}
	{\@Ponearg{#1}}
}
\newcommand{\@Pnoargs}{\mathbb{P}}
\newcommand{\@Ponearg}[1]{\mathbb{P}\left[ #1 \right]}
\newcommand{\@Ptwoargs}[2]{\mathbb{P}_{#1}\left[ #2 \right]}
\newcommand{\E}{%
	\@ifnextchar\bgroup%
	{\@Ewithargs}
	{\@Enoargs}
}
\newcommand{\@Ewithargs}[1]{%
	\@ifnextchar\bgroup%
	{\@Etwoargs{#1}}
	{\@Eonearg{#1}}
}
\newcommand{\@Enoargs}{\mathbb{E}}
\newcommand{\@Eonearg}[1]{\mathbb{E}\left[ #1 \right]}
\newcommand{\@Etwoargs}[2]{\underset{#1}{\mathbb{E}}\left[ #2 \right]}
\makeatother
\newcommand{\mE}{\mathcal{E}}

\begin{document}

\maketitle

We give an efficient algorithm for learning a binary function in a given class $C$ of bounded VC dimension, with training data distributed according to $P$ and test data according to $Q$, where $P$ and $Q$ may be arbitrary distributions over $X$. This is the generic form of what is called \textit{covariate shift}, which is impossible in general as arbitrary $P$ and $Q$ may not even overlap. However, recently guarantees were given in a model called PQ-learning \citep{GKKM:2020} where the learner has: (a) access to unlabeled test examples from $Q$ (in addition to labeled samples from $P$, i.e., semi-supervised learning); and (b) the option to \textit{reject} any example and abstain from classifying it (i.e., selective classification). The algorithm of \cite{GKKM:2020} requires an (agnostic) noise-tolerant learner for $C$.
The present work gives a polynomial-time PQ-learning algorithm, called \textit{\algfullname}, that uses an oracle to a ``reliable'' learner for $C$, where reliable learning \citep{KKM:2012} is a model of learning with one-sided noise. Furthermore, this reduction is optimal in the sense that we show the equivalence of reliable and PQ learning.

\section{Introduction}
Consider learning a binary function $f: X \rightarrow \zo$ in a given class $C$
of bounded VC dimension, with training data distributed according to $P$ and
test data according to $Q$, where $P$ and $Q$ may be arbitrary distributions
over $X$. This form of what is often called learning with \textit{Covariate
Shift} (CvS) is extreme because $P$ and $Q$ may be arbitrary, whereas much work
assumes bounded $Q(x)/P(x)$. In standard supervised learning, learning with
arbitrary CvS is known to be impossible \citep[e.g.,][]{ImpossibleCvS}. In recent work,
\cite{GKKM:2020} show that it is possible using: 
\begin{enumerate}[label=(\alph*)]
	\item Access to unlabeled test examples from $Q$ (in addition to labeled
		samples from $P$). This is often called semi-supervised learning. 
	\item The option to \textit{reject} any example and abstain from classifying
		it. This is often called \textit{selective classification} or
		classification with a reject option. Equivalently, one can think of a
		classifier that outputs $0$, $1$, or $\bot$, where $\bot$ indicates
		rejection.
	\item An efficient Empirical Risk Minimization oracle (ERM) which can find a
		classifier $c \in C$ of minimal error with respect to any training set.
		For $C$ of bounded VC dimension, this is equivalent to ``proper agnostic
		learning'' as defined by \cite{Kearns92towardefficient}, where agnostic
		learning is a model of learning with arbitrary label noise.
\end{enumerate}
\cite{GKKM:2020} call their model PQ-learning, and it is motivated by the
numerous applications where the test distribution is different than the
training distribution, for adversarial or natural reasons. Both (a) and (b) are
provably necessary for the error and rejection requirements of PQ-learning,
defined below. However, it is not clear what the \textit{computational}
requirements are---is PQ-learning as hard as agnostic learning or is it as easy
as PAC learning?

We show PQ-learning is equivalent to learning with one-sided arbitrary noise,
defined as \textit{reliable learning} by \cite{KKM:2012}. Crucial to
this result, and perhaps the most interesting part of the paper, is the \textit{\algfullname} (\alg) selective classification algorithm.  \alg uses a  reliable learner (rather than a full agnostic learner) to efficiently PQ learn. Conversely, we show that one cannot further reduce PQ-learning to a weaker oracle in that we also present a reduction from reliable learning to PQ-learning. Ignoring
computation, the number of examples required for learning in all these models
was known to be polynomially related to $d=\VC(C)$, the VC dimension of $C$. 

Further, we give evidence that the difficulty of PQ-learning, and thus also
reliable learning, lies somewhere in between that of PAC and agnostic learning,
assuming the hardness of learning parity with noise and DNFs. In particular,
we observe that parity functions are PQ-learnable. This suggests that
PQ-learning is easier than agnostic learning since there is no known
noise-tolerant parity learning algorithm, and in fact multiple cryptography
systems rely on its hardness \citep[see e.g.,][]{pietrzak2012cryptography}. We
also observe that conjunctions, which are easily PAC learnable, are unlikely to be
PQ-learnable, or at least that PQ-learning conjunctions would imply
PAC-learning DNFs, a longstanding open PAC learning problem. Hence,
there is an interesting computational complexity hierarchy in that PAC learning
is easier than reliable-learning and PQ-learning, which in turn are easier than
agnostic learning, assuming that parity is hard to agnostically learn and
DNFs are hard to PAC-learn. 

Previous work has shown such separations in related models. For example,
\citet{Bshouty2005MaximizingAW} show that learning conjunctions in a
model very similar to reliable learning implies PAC learning DNFs.
\citet{kanade2014distribution} give an algorithm for reliably learning
majorities over $\zo^d$ in time $2^{\tilde{O}(\sqrt{d})}$, whereas for this
problem there are no agnostic learning algorithms known that run in time less
than $2^{\Omega(d)}$. Although it doesn't neatly fit in the boolean function
setting, \citet{goel2017reliably} showed that the class of ``ReLUs'' over
$\zo^d$ is at least as hard to learn as learning $\omega(1)$-size parities with
noise. 

We now describe the learning models and our algorithms. \alg is intuitive and would be easy to implement in practice using off-the-shelf classifiers.

\subsection{PQ-learning}

Recall that the goal is to learn an unknown $f \in C$, where $C$ is a given
family of binary functions, with respect to arbitrary distributions $P, Q$ over
$X$. The learner is given examples from $P$ labeled by an arbitrary $f \in C$,
and unlabeled examples from $Q$. It outputs a \textit{selective classifier}
$h:X \rightarrow \{0,1,\bot\}$, and we say that $x$ is \textit{rejected} if
$h(x)=\bot$. An \textit{error} is a misclassified example that is not rejected,
i.e., $h(x)=1-f(x)$. Of course, one can guarantee 0 errors by simply rejecting
everything, or 0 rejections by classifying everything as 0, but the challenge
is to simultaneously achieve low error and rejection rates.  

Now, one can consider the rejection and error rates with respect to $P$ or $Q$.
PQ-learning requires, with high probability, at most $\eps$ error rate with
respect to $Q$, and at most $\eps$ rejection rate \textit{with respect to} $P$. 
At first this may seem counter-intuitive, as one may care only about $Q$. 
Ideally, one would have liked a low rejection rate with respect to
$Q$, but this is impossible in general since $P$ and $Q$ may be very different.
However, if one is concerned with rejection rate with respect to $Q$, an $\eps$ $P$-rejection rate implies a $Q$-rejection rate of at most 
$\eps$ plus the statistical distance between $P$ and $Q$. Thus, if $P=Q$ the rejection rate from $Q$ is at most $\eps$ and the bound degrades naturally with the degree of overlap between $P$ and $Q$. Also, as is standard, PQ learning requires the above for every $\eps>0$ with a runtime (and thus also
the number of labeled and unlabeled examples it uses) that is polynomial in
$1/\eps$. 

\subsection{Reliable learning}

Reliable learning, as defined by \cite{KKM:2012}, is a model of learning with
one-sided agnostic noise. Reliable learning is motivated by applications where
false positives (or false negatives) are to be avoided at all cost. The
reliable model applies to a standard agnostic setting (supervised learning with
$P=Q$, no unlabeled test examples, no reject option, and arbitrary $f$, even $f
\not\in C$). For the moment, suppose that $C$ is closed under complements,
meaning that for any classifier $c \in C$, $1-c$ is also in $C$. Then reliably
learning $C$ means finding a classifier $h:X \rightarrow \zo$ with $\eps$ false
positive rate and which has an error rate at most $\opt+\eps$ where $\opt$ is
the error rate of the best classifier with 0-false positive rate; assuming the constant $0$ and $1$ functions are in $C$ such a classifier always exists. (An analogous
notion for false negatives is equivalent if $C$ is closed under complements.)
Reliable learning has been applied in various works as described in the related
work section below.

In practice, it is straightforward to implement a reliable learner using any standard
classifier by a variety of means: one can heavily up-weight the negative
examples, subsample the positive examples, use a cost-sensitive classification
algorithm with high weight on false positive errors, or simply apply a positive
label only on the examples that a classifier is most confident on for classifiers
that also provide a confidence signal such as a margin. Thus there are a
variety of ways to implement Algorithm \ref{alg:PQtoReliable} in practice.
Formally, \cite{KKM:2012} prove that reliable learning is no harder than
agnostic learning. 

\subsection{The \algfullname (\alg) Algorithm}
The key question in selective classification and PQ learning is what to reject. Like previous selective classification algorithms, \alg first trains a classifier $c$ on the labeled training data from $P$ using, say, a PAC-learner. Like many such algorithms, \alg outputs a selective classifier $h$ such that $h(x)\in \{c(x),\bot\}$, i.e., its classifications agree with $c$ except that it rejects some examples. 

To determine what to reject, \alg ``slices'' the space $X$ into two parts: where $c(x)=0$ versus $c(x)=1$. It then rejects examples from each of these parts separately by repeated ``dicing'': distinguishing examples that clearly come from $Q$ versus those that may come from $P$, on the respective part. To dice on $c(x)=0$, it creates an artificial datasets of examples where $c(x)=0$ consisting of both $P$-examples, labeled $0$ and $Q$-examples, labeled 1. It then trains an sequence of positive reliable learners to distinguish $P$ from $Q$, and $h$ rejects examples that are classified as 1 by any such distinguisher. Now, not all examples from $Q$ can be clearly distinguished from $P$ (e.g., where say $Q(x)\leq 2 P(x)$) but it turns out that not all examples need to be distinguished. It suffices to reject examples from $Q$ that are clearly distinguishable from $P$ by classifiers from $C$, but one must take care not to reject examples that are in fact from $P$ (few false positives, hence positive reliable learning). A similar approach is used on the $c(x)=1$ part. 

The idea of rejecting examples by training a distinguisher to distinguish examples from $Q$ versus $P$ is intuitive. Unfortunately, this approach does not directly work without slicing the space into parts where $c(x)=0$ and $c(x)=1$. This is illustrated by a trivial 4-point halfspace example in Fig.~\ref{fig:intro}(a),
where $C$ is the class of homogeneous halfspaces $\sgn(w_1 x_1 + w_2 x_2)$ that pass through the origin in two dimensions; $P$ is uniform over two symmetric unit vectors $\{u, -u\}$; and $Q$ is uniform over $\{u, -u, v, -v\}$, the same two points plus two additional symmetric unit vectors. Clearly $\pm v$ must be rejected because their labels are not determined by those of $u$, while we must not reject $\pm u$, but no homogeneous halfspace can distinguish $\pm u$ from $\pm v$. However, once one slices the space into positive and negative components, $v$ can be distinguished from $u$ within the respective parts. 
Fig.~\ref{fig:intro}(b) illustrates \alg for halfspaces more generally.

In contrast, the Rejectron algorithm of \cite{GKKM:2020} 
finds a sequence of candidate alternative classifiers that agree with the training data but disagrees with the test data and reject the disagreement regions. While this approach is also intuitive, it requires the full power of ERM (i.e., agnostic learning) rather than reliable learning and thus may require greater resources. 

\begin{figure}
	\begin{tabular}{cc}
	\begin{minipage}{0.39\textwidth}
		\centering
		\begin{tikzpicture}[scale=0.8]
	\def\cosfortyfive{0.707107}
	\def\costhirty{0.866025}
	\def\sinthirty{0.5}
	\draw (0, 0) circle (2cm);

	\draw (2*\cosfortyfive, 2*\cosfortyfive) node {\Large{\textcolor{red}{$\textbf{+}$}}};
	\draw (2*\cosfortyfive, 2*\cosfortyfive) node [above right] {$u$};
	\draw (-2*\cosfortyfive, -2*\cosfortyfive) node {\Large{\textcolor{blue}{$\textbf{--}$}}};
	\draw (-2*\cosfortyfive, -2*\cosfortyfive) node [below left] {$-u$};
	\draw[fill=black] (2*\costhirty, -2*\sinthirty) circle (0.07cm);
	\draw[fill=black] (2*\costhirty, -2*\sinthirty) node [below right] {$v$};
	\draw[fill=black] (-2*\costhirty, 2*\sinthirty) circle (0.07cm);
	\draw[fill=black] (-2*\costhirty, 2*\sinthirty) node [above left] {$-v$};
	\draw[fill=black] (0, 0) circle (0.07cm);
	\draw[-, thick] (-2.5*\cosfortyfive, 2.5*\cosfortyfive) -- (2.5*\cosfortyfive, -2.5*\cosfortyfive) node [right] {$c$};
\end{tikzpicture}
	\end{minipage} & 
	\begin{minipage}{0.59\textwidth}
		\centering
		\definecolor{darkgreen}{rgb}{0.0, 0.5, 0.13}
		\begin{tikzpicture}[scale=0.8]
	\def\hlx{-1}
	\def\hly{3}
	\def\hrx{1}
	\def\hry{-3}
	\def\calx{-0.67}
	\def\caly{2}
	\def\carx{0.6}
	\def\cary{3.2}
	\def\cblx{-0.7}
	\def\cbly{2.1}
	\def\cbrx{2}
	\def\cbry{-3}
	\def\cclx{1.2}
	\def\ccly{-2.8}
	\def\ccrx{2.4}
	\def\ccry{3}

	\foreach \x/\y in {0.1/1.7, 1/1, 2/2, 1.5/3, 1.2/-1, 1.5/-0.5, 1/-0.5}
		\draw (\x,\y) node {\Large{\textcolor{red}{$\textbf{--}$}}};
	\foreach \x/\y in {0.6/1.9, 1.7/1.1}
		\draw[fill=black] (\x, \y) circle (0.04cm);

	\foreach \x/\y in {-0.5/-1.7, -1.2/0.4, -2/0.2, -1.5/0, -1.2/-1, -1.5/-0.5, -1/-0.5}
		\draw (\x,\y) node {\Large{\textcolor{blue}{$\textbf{+}$}}};
	\foreach \x/\y in {-1.3/-0.1, -1.7/-1.1}
		\draw[fill=black] (\x, \y) circle (0.04cm);

	\foreach \x/\y in {2.1/0.7, 2/0.3, 2.5/0.8, 2.5/1.3, 2.7/-1, 2.5/-0.5, 2.2/-0.6}
		\draw[fill=black] (\x, \y) circle (0.04cm);

	\foreach \x/\y in {-0.7/1.1, -0.2/1, -0.5/0.8, -0.5/1.3, 0.7/-1, 0.5/-0.5, 0.25/-0.6, -0.5/1.2, -0.6/1.3, -0.2/-0.2, -0.22/-1, -0.3/-0.8}
		\draw[fill=black] (\x, \y) circle (0.04cm);

	\foreach \x/\y in {-0.7/2.1, -1.2/2.3, -0.5/2.8, -0.5/2.3, -1/2.5, -0.5/2.4, -1.2/2.6, -0.5/3.2, -0.6/2.3, -0.2/3.2, -0.22/3, -0.3/2.8}
		\draw[fill=black] (\x, \y) circle (0.04cm);

	\draw[-,very thick] (\hlx, \hly) -- (\hrx, \hry) node [below] {$c$};
	\draw (\hrx + 1, \hry) node [right] {$c(x) = 0$};
	\draw[-, thick, darkgreen] (\calx, \caly) -- (\carx, \cary);
	\draw[-, thick, darkgreen] (\cblx, \cbly) -- (\cbrx, \cbry);
	\draw[-, thick, darkgreen] (\cclx, \ccly) -- (\ccrx, \ccry);

	\usetikzlibrary{patterns}
	\fill[pattern color=darkgreen!50, pattern=vertical lines] (\calx, \caly) -- (\carx, \cary) -- (\carx, \cary + 0.3 ) -- (\hlx-0.1, \hly + 0.3);
	\fill[pattern color=darkgreen!50, pattern=north west lines] (\cblx, \cbly) -- (\cbrx, \cbry) -- (\hrx, \hry);
	\fill[pattern color=darkgreen!50, pattern=horizontal lines] (\cclx, \ccly) -- (\ccrx, \ccry) -- (\ccrx + 0.5, \ccry) -- (\ccrx + 0.5, \ccly);
\end{tikzpicture} 
\end{minipage}\\
(a) & 

	 (b)
	\end{tabular}
	\caption{(a) When learning homogeneous halfspaces (passing through the origin), there may be no halfspace that separates training examples $\pm u$ from test examples $\pm v$ even though the test examples must be rejected because the labels of $\pm u$ give no information about the labels of $\pm v$. Rejecting all points is not possible as we must not reject the points $\pm u$. 
	(b) An illustration of \alg for general halfspaces that do not necessarily pass through the origin. When focusing on examples with $c(x)=0$, we iteratively find classifiers in $C$ that reliably separate the problematic $x \sim Q$ (which we will label as positive) from $x \sim P$ (which retain their labels and will be mostly negative) and reject the green-striped regions. By only focusing on the $c(x)=0$ part we avoid having to find classifiers that also correctly classify the true positive examples.}
    \label{fig:intro}
\end{figure}
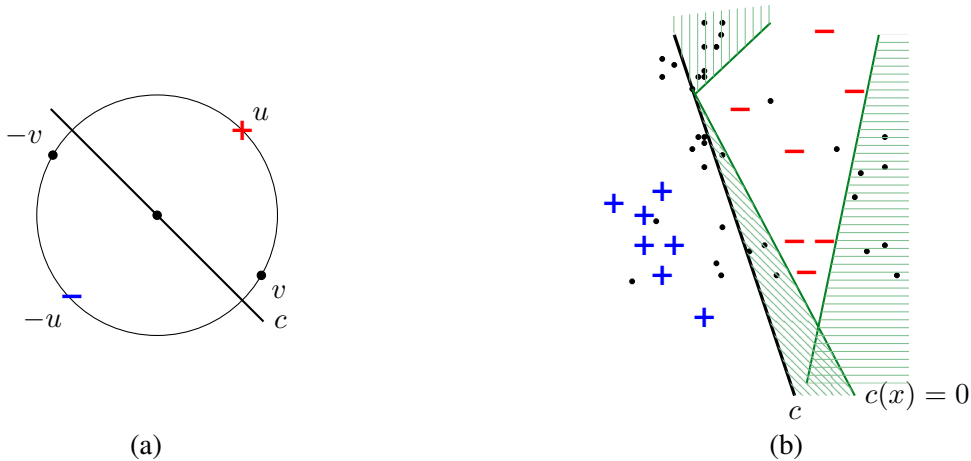

Our solution is simple. Like the ``Rejectron'' algorithm of \cite{GKKM:2020}, we first train a classifier $h$ on the labeled training data using, say, a PAC-learner. Then we separately focus on examples satisfying $h(x)=1$ and $h(x)=0$. Among those examples with $h(x)=1$, we show that there must exist $c \in C$ that distinguishes at least some of the problematic $x \sim Q$ from $x \sim P$. And finding such a classifier amounts to a reliable classification problem, i.e., learning with one-sided noise. 

\subsection{Related work}
This work is intimately related to that of \cite{GKKM:2020} and \cite{KKM:2012}. \cite{GKKM:2020} introduces both PQ-learning and an adversarial transductive model of learning and exhibits a trade-off between rejection rates and accuracy that we do not study here. Reliable learning has also been studied within the learning theory community for a variety of applications \citep[see e.g.,][]{kanade2014distribution, goel2017reliably, durgin2019hardness}. Several other models of learning are similar to reliable learning including models of \cite{pitt88, juba2016learning}. In a related model, \cite{Bshouty2005MaximizingAW} show that learning conjunctions would imply learning DNF.  

In supervised learning with CvS, a body of related work within the learning
theory community studies learning with CvS assuming that $Q(x) \leq M \cdot
P(x)$ for all $x \in X$ and some constant
$M>1$~\citep[e.g.,][]{huang2007correcting,ben2012hardness}. And without such an
assumption supervised learning has been shown to be impossible
\citep[e.g.,][]{ImpossibleCvS}. 

A separate body of work studies selective classification~\citep{rivest1988learning,li2011knows, sayedi2010trading}. Some of that work, particularly the work in online algorithms with a reject option, were targeted at non-stationary sequences of examples $x$. Even intervals are impossible to learn in online models, and in a supervised iid model \cite{kivinen1990reliable} showed that exponentially many examples are required to learning rectangles under uniform distributions (as cited by \cite{hopkins2019power, GKKM:2020}). Part of the challenge is that most definitions also require few test rejections, unlike PQ-learning's requirement of few rejections with respect to $P$. Other work, including the work of \cite{KKM:2012}, makes the standard assumption that $P=Q$ and instead uses rejections to handle uncertain regions due to agnostic noise. 

Finally, semi-supervised learning has been studied extensively in learning theory, again generally where $P=Q$. However, in the worst case the addition of unlabeled examples does not seem to provide significantly better guarantees than supervised learning \citep[e.g.,][]{ben2008does}.

\section{Preliminaries and Background}
\label{sec:prelims}
Let $X$ be the input space representing unlabeled examples and $Y = \{0, 1\}$ the target labels. Let $Y^X$ denote the set of functions from $X \rightarrow Y$. For $f, h \in Y^X$, and a probability measure $P$ over $X$, we denote the error of a hypothesis $h$ with respect to a ground truth classifier $f$ by 
\begin{align*}
	\err_P(h; f) = \P{x \sim P}{h(x) \neq f(x)}.
\end{align*}
A \emph{selective classifier} is a function $h : X \rightarrow \{0, 1, \question \}$. For such a classifier $\err_P(h; f)$ is defined as,
\begin{align*}
	\err_P(h; f) = \P{x \sim P}{h(x) \neq f(x) \wedge h(x) \neq \question}.
\end{align*} 
We can also define the \emph{rejection} rate of $h$ with respect to a distribution $P$ as
\begin{align*}
	\rej_P(h) = \P{x \sim P}{h(x) = \question }.
\end{align*}
As observed by \cite{GKKM:2020}, the rejection rate with respect to $P$ can be used to bound the rejection rate with respect to $Q$ in multiple ways, including the following:
\begin{align*}
	\rej_Q(h) = \P{x \sim Q}{h(x)=\question} \leq \rej_P(h) + \|P-Q\|_{TV},
\end{align*}
where $\|P-Q\|_{TV}$ is the total variation distance (also called statistical distance) between $P$ and $Q$.  

We also denote the false positive rate of $h$ with respect to a distribution $D$ over $X \times \{0, 1\}$ by
\begin{align*}
	\false^+_D(h) = \P{(x, y) \sim D}{h(x) = 1 \wedge y = 0}.
\end{align*}
Similarly the false negative rate of $h$ with respect to a distribution $D$ over $X \times \{0, 1\}$ is denoted by 
\begin{align*}
	\false^-_D(h) = \P{(x, y)\sim P}{h(x) = 0 \wedge y = 1}.
\end{align*}

A concept class $C \subseteq Y^X$ is a collection of functions from $X
\rightarrow Y$. For classifier $c$, we denote the complementary classifier by $\bar{c} = 1-c(x)$ and we denote $\bar{C}=\{\bar{c}~|~c \in C\}$. 
Following \citet{KKM:2012}, we assume that the constant 0 and 1 functions are both in $C$ in order for their notions of positive and negative reliable learning to be well-defined.

A distribution $P$ over $X$ and a function $f : X \rightarrow
[0, 1]$ define a joint probability distribution $D_{P, f}$ over $X \times \{0,
1\}$, such that the marginal distribution over $X$ is $P$ and $\E{(x, y) \sim
D_{P, f}}{y | x} = f(x)$. Using this notation, the error rate of a classifier (regular or selective) can be decomposed into false positives and false negatives:
\begin{align}\label{eq:decomposition}
	\err_P(h, f) = \false^+_{D_{P, f}}(h) + \false^-_{D_{P, f}}(h).
\end{align}

For a distribution $P$ over $X$ and a labeling function $f : X \rightarrow [0,
1]$, we denote by $\EX(P, f)$ an example oracle that when queried gives a
random \emph{labeled} example from $D_{P, f}$.  For distribution $D$ over $X
\times \{0,1\}$, we similarly define oracle $\EX(D)$.  And for a distribution
$P$ over $X$, we denote by $\EX(P)$ an example oracle that when queried gives a
random \emph{unlabeled} example drawn from the distribution $P$.  By a slight
abuse of notation, $\alpha \EX(P_1, f_1) + (1 - \alpha) \EX(P_2, f_2)$ denotes
an oracle that returns a labeled example $(x, y)$ where with probability
$\alpha$, $(x, y)$ is drawn from $D_{P_1, f_1}$, and with probability $1 -
\alpha$, $(x, y)$ is drawn from $D_{P_2, f_2}$. Finally, for an event $A$,
$P|_A$ denotes the probability distribution $P$ conditioned on $A$. 

\subsection{PQ Learning} 

We define the notion of PQ-Learning introduced in the recent work
by~\citet{GKKM:2020}. The definition below is mathematically
identical to theirs, but presented slightly differently using the oracles
$\EX(P, f)$ and $\EX(Q)$. 

\begin{definition}[PQ Learning~\citep{GKKM:2020}] A concept class $C$ over $X$ is PQ-learnable, if there exists a learning algorithm $L$, such that for every pair of distributions $(P, Q)$ over $X$, every target $f \in C$, for every $\epsilon > 0$ and for every $\delta > 0$, when given access to the labeled example oracle $\EX(P, f)$ and the unlabeled example oracle $\EX(Q)$, $L(\epsilon, \delta, \EX(P, f), \EX(Q))$ outputs a selective classifier $h : X \rightarrow \{0, 1, \question \}$, that with probability at least $1 - \delta$ simultaneously satisfies $\err_Q(h; f) \leq \epsilon$ and $\rej_P(h) \leq \epsilon$. Furthermore, $L$ must run in time polynomial in $1/\epsilon$ and $1/\delta$. 
\end{definition}

\subsection{Reliable Learning}

We now recall the definition of full reliable agnostic learning, introduced by~\citet{KKM:2012}, which we refer to as \textit{reliable learning} for brevity. Formally, this is defined in terms of their notions of positive and negative (agnostic) reliable learning, which capture robustness to one-sided noise. However, it is worth noting that the three definitions are all equivalent for concept classes $C=\bar{C}$ that are closed under complements.

\begin{definition}[Positive Reliable Learning] A concept class $C$
	over $X$ is positive reliably learnable, if there exists a
	learning algorithm $L$, such that for every distribution $D$ over
	$X \times \{0,1\}$, every $\epsilon > 0$ and
	$\delta > 0$, when given access to the labeled example oracle $\EX(D)$, $L\left(\epsilon, \delta, \EX(D)\right)$ outputs a hypothesis $h : X \rightarrow \{0,
	1\}$, that with probability at least $1 - \delta$ satisfies $\false^+_D(h) \leq \epsilon$ and $\false^-_D(h) \leq \opt_+ +
	\epsilon$, where,
	\begin{align*}
		\opt_+ = \min_{c \in C:~ \false^+_D(c) =
		0}\false^-_D(c). 
	\end{align*}
	Furthermore, we require the running time of $L$ to be polynomial in $1/\epsilon$ and $1/\delta$. 
\end{definition}
A class $C$ is said to be \textit{negative reliably learnable} if $\bar{C}$ is positive reliably learnable, which optimizes subject to a restriction on false negatives. As mentioned, it must be assumed that the constant 0 (1) function is in $C$ in order for positive (negative) reliable learning to be well-defined in the case where $f=0$ ($f=1$). Thus, as in \cite{KKM:2012}, we assume that both the constant 0 and 1 functions are in $C$. 
We now define our main notion, reliable learning; it is easy to see that it is equivalent to positive reliable learning if $C=\bar{C}$.
\begin{definition}[Reliable Learning] A concept class $C$
	over $X$ is reliably learnable if $C$ and $\bar{C}$ are both positive reliably learnable.
\end{definition}

We note that the above notion is equivalent to the notion of \emph{fully reliable agnostic learning} from \citet{KKM:2012}. To define that notion, they considered a learning algorithm that outputs a selective classifiers $h$ which has both false positive and negative rates
bounded by $\epsilon$. A pair of classifiers $(c_+, c_-)$, where both $c_+, c_-
\in C$, satisfying $\false^+(c_+) = 0$ and $\false^-(c_-) = 0$, can be
converted to a reliable selective classifier which has no false positive or
negative errors: output $c_+(x)$ if $c_+(x) = c_-(x)$ and output $\question$
otherwise. Furthermore, it is required that the probability $\P{h(x) =
\question} \leq \opt_\question + \epsilon$, where $\opt_\question$ is the
probability of predicting $\bot$ for the best pair of classifiers $(c_+, c_-)$
as defined previously.  To see that reliable learning is equivalent to fully reliable agnostic learning, note that given both positive and negative reliable learners, it is straightforward to construct an $h$ with $\P{h(x) =
\question} \leq \opt_\question + 2\epsilon$. And conversely, given a fully reliable learner, it is straightforward to construct positive and negative learners by simply converting its $\bot$ predictions to 0 or 1, respectively. 

\section{Equivalence Results}
In this section, we prove our main result showing the equivalence between
reliable learning and PQ learning up to polynomial factors in the
running time (and sample complexity). Theorem~\ref{thm:PQtoReliable} states that any concept class $C$ that is reliably learnable is also PQ-learnable. Theorem~\ref{thm:ReliabletoPQ} states the converse.

\ifarxiv
\begin{algorithm}
\else
\begin{algorithm2e}
\fi
	\ifarxiv\else\label{alg:PQtoReliable}\fi
	\caption{\ifarxiv\label{alg:PQtoReliable}\fi The \algneg algorithm which takes as input a classifier $c$ with error $\leq \eps/2$ and a positive reliable learning algorithm. It rejects part of the $c(x)=0$ region so as to ensure a small false negative rate with respect to $Q$ and small reject rate with respect to $P$.}
	\SetKwInput{KwIn}{Inputs}
	\SetKwInput{Return}{Return}
	\SetKw{And}{and}
	\SetKw{Break}{break}
	\DontPrintSemicolon
	\LinesNumbered
	\KwIn{\;
		~~a. Accuracy parameter $\epsilon$, Confidence parameter $\delta$ \;
		~~b. Access to oracles $\EX(P, f)$, $\EX(Q)$ \;
		~~c. Classifier $c: X \rightarrow \zo$ \;
		~~d. Positive reliable learner $L$ \;  }
		\textbf{let} $M = \frac{2}{\epsilon}\cdot \log\frac{1}{\epsilon}$\;
	\If{\{$\P{x \sim P}{c(x) = 0} < \epsilon$\}\label{algline:ifNegligible}}{
		    \textbf{define} $h(x) = 
        	\begin{cases} 
        		c(x) & \text{if } c(x) = 1\\
        		\question & \text{otherwise}
        	\end{cases}$ \;
        	\Return{$h$}
		}
	\textbf{define the event} $\mE_1 = \{x \in X\mid c(x) = 0\}$ \;
	\textbf{let} $i = 1$\;
	\While{$\{\P{x \sim Q}{x \in \mE_i} > \epsilon\}$ \label{algline:whilePos}}{
		$c_i = L\left(\frac{\epsilon}{2M}, \frac{\delta}{M}, \frac{1}{2}\EX(P|_{\indicator(c(x) = 0)}, f) + \frac{1}{2} \EX(Q|_{\mE_i}, 1)\right)$ \;\label{algline:NegRelLearn}
		\If{\{$\P{x \sim Q|_{\mE_i}}{c_i(x) = 1} < \frac{\epsilon}{2}$\}\label{algline:ifPos}}{
			\Break\; 
		}
		$\mE_{i+1} = \{ x \in \mE_i \mid  c_i(x) = 0\} = \{x \in X \mid c(x)=c_1(x)=\ldots=c_i(x)=0\}$ \;
		$i = i + 1$ \;
	} 
	\textbf{define} $h(x) = 
	\begin{cases} 
		c(x) & \text{if } c(x) = 1 \text{ or } \displaystyle\bigwedge\nolimits_i c_i(x) = 0\\
		\question & \text{otherwise}
	\end{cases}$ \;
	\Return{$h$}
\ifarxiv
\end{algorithm}
\else
\end{algorithm2e}
\fi

Before presenting the formal proof which relies on a slightly delicate argument, we give a high-level idea of the proof. In the slice step, \alg first finds a classifier $c$ that PAC-learns the target $f$ under the distribution $P$. This can be done since, without noise, a positive reliable learner (or a negative reliable learner) is also a PAC learner. The final classifier $h$ will never have opposite labels with $c$ but it may reject some examples, i.e., $h(x) \in \{c(x), \question\}$ for all $x \in X$. Now we separately consider the parts where $c = 0$ and $c = 1$. Since they are symmetric, we'll focus on the part when $c = 0$. Our \algneg algorithm will distinguish $Q$ from $P$, or at least where necessary. To do so, we will construct an artificial dataset and call a reliable learner on it. The dataset will consist entirely of examples where $c(x)=0$, and subject to that condition it will be an equal mixture of examples drawn from $P$ and from $Q$. The $P$ examples will be labeled by $f$, which means that most, though not necessarily all, will also have $f(x) = 0$. We will label the examples drawn from $Q$ as $1$. Then we will find a \emph{positive reliable classifier} for the resulting distribution. This allows us to identify a region that is almost exclusively $Q$, which we decide to reject. We repeat this process iteratively rejecting more and more of the space until we cannot find a nontrivial region to safely reject. The key observations are that: (a) the ground truth classifier $f$ is always reliable, and thus we will continue to reject a non-trivial fraction of $Q$ as long as there remain remains a non-negligible region under $Q$ where $f$ could output $1$; and (b) each iteration rejects a region of small probability under $P$.  In the end, we (mostly) only output $0$ on examples under $Q$ for which $f$ is actually $0$. The \algpos algorithm is completely analogous and runs on $c(x)=1$ using negative reliable learning. 

\begin{lemma}\label{lem:PQneg}
	For any $\eps, \delta>0$, distributions $P, Q$ over $X$, $f \in C$, and $c: X \rightarrow \zo$, with probability $\geq 1-\delta$, Algorithm \ref{alg:PQtoReliable} (\algneg) returns a classifier $h$ with $h(x) \in \{\question, c(x)\}$ for all $x \in X$, $\rej_P(x) \leq \eps + \err_P(c; f)$, and $\false^-_{D_{Q, f}}(h) \leq \eps$. The algorithm runs in expected time polynomial in $1/\eps, 1/\delta$ assuming access to an exact probability computation oracle required in Lines~\ref{algline:ifNegligible}, \ref{algline:whilePos} and \ref{algline:ifPos}.
\end{lemma}
In other words, \algneg guarantees a low false positive rate by rejecting examples. Completely analogously, one defines \algpos and argues that used with a negative reliable learner, can guarantee a low false negative rate by rejecting further examples. The two can be applied to guarantee a low total error (by Eq.~\eqref{eq:decomposition}) with bounds on the rejection rate with respect to $P$. Before we prove the above lemma, we show how to use it to prove the main theorem. The access to an exact probability computation oracle is not necessary, and estimates of the require probabilities, which can be obtained by sampling, are sufficient after minor adjustments to the constant factors, as explained in the proof of Theorem~\ref{thm:PQtoReliable}.

\begin{theorem} If a concept class $C$ is reliably learnable, then $C$ is PQ-learnable.  \label{thm:PQtoReliable}
\end{theorem}
\begin{proof}
We first find $c$ with $\err_P(c; f) \leq \eps/4$ by running a positive (or negative) reliable learner with parameters $\eps/8, \delta/4$ on ordinary labeled examples $\EX(P, f)$. With probability $\geq 1-\delta/4$ its output $c$ has both false positive and false negative rates of $\eps/8$ (since $\opt_+=0$). By the error decomposition in Eq.~\eqref{eq:decomposition}, this means it has error at most $\eps/4$. 

Next, apply Algorithm \ref{alg:PQtoReliable} (\algneg) to $c$ with parameters $\eps/4$ and $\delta/4$. By Lemma \ref{lem:PQneg}, its output $h_+$ guarantees at most $\eps/4 + \eps/4 =\eps/2$ false negative rate with respect to $Q$ and a rejection rate of  at most $\eps/4$ with respect to $P$ (with probability $\geq 1-\delta/4$). Then run the symmetric version of Algorithm \ref{alg:PQtoReliable} (\algpos) with a negative reliable learner to receive $h_-$ with false positive rate at most $\eps/2$ with respect to $Q$, and a rejection rate of at most $\eps/4$ with respect to $P$ (with probability $\geq 1-\delta/4$). Finally, reject points that are rejected by either classifier, i.e., outputting the classifier: 
\begin{align*}
    h(x) = \begin{cases} 
		h_+(x) & \text{if } h_+(x) = h_-(x)\\
		\question & \text{otherwise}.
	\end{cases}
\end{align*}
By the union bound, it is easy to see that the rejection rate with respect to $P$ would be at most $\eps$. Further, by Eq.~\eqref{eq:decomposition}, the error rate with respect to $Q$ would be at most $\eps$, as required.

Unfortunately, Algorithm \ref{alg:PQtoReliable} is a hypothetical algorithm since it requires \textit{exact} probability calculations in Lines~\ref{algline:ifNegligible}, \ref{algline:whilePos} and
\ref{algline:ifPos}. In reality, the probabilities can only be estimated to
a high accuracy with high probability. Using standard tedious arguments, one can design an algorithm without an exact probability oracle that enjoys the same guarantees of Lemma~\ref{lem:PQneg}. It would involve a change of constants, including running the positive reliable learner with parameters that are a constant factor smaller, and straightforward applications of the Chernoff-Hoeffding bound and the union bound. 

A second issue with Algorithm \ref{alg:PQtoReliable} is that it only runs in expected polynomial runtime, i.e. a ``Las Vegas'' algorithm. However, a standard timeout approach can be used to convert any Las Vegas algorithm into one that certainly runs in polynomial time (i.e. a ``Monte Carlo'' algorithm) and produces an identical result with probability $\geq 1-\delta/4$.  
\end{proof}

We now prove Lemma \ref{lem:PQneg}.
\begin{proof}[\ifarxiv Proof \fi of Lemma \ref{lem:PQneg}]
First note that the lemma holds trivially if the condition in Line~\ref{algline:ifNegligible} holds, as rejecting all negative examples will necessarily lead to a false negative rate of 0 and a rejection rate of at most $\eps$ under $P$. So henceforth consider the case in which $\Pr_{x \sim P}[c(x)=0] \geq \eps$.

For $i \geq 0$, let $q_i = \P{x \sim Q}{x \in \mE_i}$. Then, provided the condition in Line~\ref{algline:ifPos} does not cause the loop on
	Line~\ref{algline:whilePos} to break, we have that $q_i \leq q_{i-1}
	\cdot (1 - \epsilon/2)$. It will also terminate if $q_i \leq \eps$. 
	Let $j$ be the final value of $i$ when the loop terminates. Thus, \begin{align*}
	      \eps \leq (1-\eps/2)^{j} \leq e^{-j\eps/2},
	\end{align*}
	or equivalently $j \leq (2/\epsilon) \log(1/\epsilon) = M$.
	
    Therefore, with probability $\geq 1-j(\delta/M) \geq 1-\delta$, all calls to the positive reliable learners succeed. Let us assume this is the happens. 
    
    We first bound the rejection rate, then the false negative rate, and finally analyze the runtime.
    Let $D_i$ denote the distribution over $X \times \{0, 1\}$ produced by the example oracle used to learn $c_i$. Since the reliable learner succeeded, for each $i$ we have,
	\begin{align*}
	    \frac{\eps}{2M} & \geq \false^+_{D_i} (c_i) \\
	    &=\frac{1}{2}\Pr_{x \sim P}[c_i(x)=1\wedge f(x)=0\mid c(x)=0] + \frac{1}{2}\Pr_{x \sim Q}[c_i(x)=1\mid \mE_i]\\
	    &\geq \frac{1}{2}\Pr_{x \sim P}[c_i(x)=1 \wedge f(x)=c(x)=0] + 0 \\
	    \frac{\eps}{M} &\geq \Pr_{x \sim P}[c_i(x)=1 \wedge f(x)=c(x)=0]
	\end{align*}
	In the above, we have used $\Pr[A\mid B] \geq \Pr[A\wedge B]$ by Bayes rule. This implies,
	\begin{align*}
		\P{x \sim P}{f(x)= 0 \wedge h(x) = \question} &= 															  \P{x \sim P}{f(x)=c(x)=0 \wedge \exists i c_i(x) = 1}  \\
		&\leq \sum_i \P{x \sim P}{f(x) = c(x) = 0 \wedge c_i(x) = 1} \\ 
		&\leq j \frac{\eps}{M} \leq M \frac{\eps}{M} =  \eps.
	\end{align*}
	Thus, as promised for the rejection bound,
	\begin{align*}
		\P{x \sim P}{h(x) = \question} &= \P{x \sim P}{f(x)= 0 \wedge h(x) = \question} + \P{x \sim P}{f(x)= 1 \wedge h(x) = \question}\\
		&\leq \eps + \P{x \sim P}{f(x)= 1 \wedge c(x) = 0}\\
		&\leq \eps + \err_P(c; f).
	\end{align*}
	
    We next bound the false negative rate, assuming that none of the positive reliable learners failed. Note that $h(x) = 0$ iff $x \in \mE_{j}$. This implies:
    \begin{align*}
        \false^-_{D_{Q, f}}(h) = 
		\P{x \sim Q}{f(x) = 1 \wedge h(x) = 0} &= \P{x \sim Q}{f(x) = 1 \wedge x \in \mE_{j}} 
    \end{align*}
	Thus to bound the false negative rate, we must show,
    \begin{align}\label{eq:toshow}
        \P{x \sim Q}{f(x) = 1 \wedge x \in \mE_{j}} \leq \epsilon
    \end{align}
	To do this, first note that if the while loop terminates because $\P{x \sim Q}{x \in \mE_{j}} \leq \epsilon$, then Eq.~\eqref{eq:toshow} holds trivially. So, suppose the loop termination is caused by the if statement on Line~\ref{algline:ifPos}.
    
    Now, by assumption $f \in C$ and by definition of $D_{j}$, $\false^+_{D_{j}}(f)=0$ and
    \begin{align*}
		\false^-_{D_{j}}(f) = \frac{1}{2} \P{x \sim Q|_{\mE_{j}}}{f(x) = 0}
	\end{align*}

	On the other hand, we have $\false^-_{D_{j}}(c_{j}) \geq \frac{1}{2} \P{x \sim Q|_{\mE_{j}}}{c_{j}(x) = 0}$ and also by the guarantee of the learning algorithm $L$, $\false^-_{D_{j}}(c_{j}) \leq \false^-_{D_{j}}(f) + \eps/(2M)$. Combining gives,
	\begin{align*}
	    \frac{1}{2} \P{x \sim Q|_{\mE_{j}}}{c_{j}(x) = 0} \leq \false^-_{D_{j}}(c_{j}) \leq \frac{1}{2} \P{x \sim Q|_{\mE_{j}}}{f(x) = 0} + \frac{\eps}{2M}
	\end{align*}
    Rearranging and using the fact that $\Pr[f(x)=1]=1-\Pr[f(x)=0]$ gives,
    \begin{align*}
	    \P{x \sim Q|_{\mE_{j}}}{f(x) = 1} &\leq \P{x \sim Q|_{\mE_{j}}}{c_{j}(x) = 1} + 2\frac{\eps}{2M} \\
	    &\leq \frac{\eps}{2} + \frac{\eps}{M} \leq \eps,
	\end{align*}
    where we have used the stopping criterion $\P{x \sim Q|_{\mE_{j}}}{c_{j}(x) = 1}\leq \eps/2$ from Line~\ref{algline:ifPos}. Thus
    \begin{align*}
	\P{x \sim Q}{f(x) = 1 \wedge x \in \mE_{j}} 
		&\leq \P{x \sim Q}{f(x) = 1 \mid \mE_{j}} \leq  \eps
    \end{align*}
    The above is what we needed for Eq.~\eqref{eq:toshow} to bound the false negative rate. Since the total failure probability is at most $\delta$, only the runtime analysis remains. 
    
    We have assumed that we have a unit-time oracle for exact probability calculations. To simulate samples from a mixed example oracle $D_i$, one flips a coin and chooses which oracle to sample from. To sample from a conditional oracle, one simply continues drawing samples until a sample satisfies the relevant event. The \textit{expected} number of samples required is the reciprocal of the probability of the event. In this case, the expected number of samples required to get a sample from $\frac{1}{2}\EX(P|_{\indicator(c(x) = 0)}, f) + \frac{1}{2} \EX(Q|_{\mE_i}, 1)$ is at most,
    \begin{align*}
        \frac{1}{\min\left\{\P{x \sim P}{c(x)=0}, \P{x \sim Q}{x\in \mE_i} \right\}}.
    \end{align*}
    But the algorithm tests in Lines~\ref{algline:ifNegligible} and \ref{algline:whilePos} ensure that the above is $O(1/\eps)$. Since the positive reliable learner is called at most $M$ times with parameters $\poly(1/\eps, 1/\delta)$, it also runs in time $\poly(1/\eps, 1/\delta)$.
\end{proof}

\begin{theorem} If a concept class $C$ is PQ-learnable, then $C$ is reliably learnable. \label{thm:ReliabletoPQ}
\end{theorem}
\begin{proof}
	We show that $C$ is positive reliably learnable; the proof of
	negative reliable learnability is obtained \textit{mutatis mutandis}.  To reduce PQ-learning to reliable learning, given labeled examples from an arbitrary distribution $D$ over $X \times \zo$, we must construct a noiseless distribution over labeled training examples and a distribution over unlabeled test examples. Our distribution over labeled training examples will simply be examples drawn from $D$ whose labels are negative, which is consistent with the all 0 classifier which we have assumed is in $C$. 

	Define the distributions $P, Q$ to be the marginal
	distributions over $X$ conditioned on $y=0$ and $y=1$ respectively. Formally, for any measurable set $A \subseteq X$, they are defined by
	$P(A) = D(A \times \{0 \})/D(X \times \{0 \})$ and $Q(A) = D(A \times \{1
	\})/D(X \times \{1\})$. Sampling from $P$ and $Q$ can be performed by
	sampling from $D$ and rejecting based on the value of $y$. The expected time to generate a sample from either distribution depends on the minimum class probability. To ensure this can be done efficiently, we suppose that $\P{(x, y) \sim
	D}{y = 0}>\eps/2$ and $\P{(x, y) \sim D}{y = 1}>\eps/2$. Otherwise
	a positive reliable learner is easily obtained by outputting a constant hypothesis. (A hypothesis with error at most $\eps$ must have both false positive and false negative rates of at most $\eps$, but we use $\eps/2$ since we cannot compute probabilities exactly.)
	
	Claim 4.2 in \cite{GKKM:2020} shows that if you can PQ-learn then you can additionally guarantee $\err_P(h; f) \leq \eps$. (This is done by simply running the PQ learner on distribution $Q'=\frac{1}{2}P + \frac{1}{2}Q$.) Let us therefore assume that we have a PQ-learner $L$ for $C$ that guarantees $\err_Q(h; f)$, $\err_P(h; f)$, and $\rej_P(h)$ are all at most $\eps$ with probability $\geq 1-\delta$. 
	
	Now, let $c^* \in C$ be a concept such
	that $\false^+_D(c^*) = 0$ and $\false^-_D(c^*) = \opt_+$. Note that unlike
	positive reliable learning, PQ-learning requires a target concept;
	we shall let $c^*$ be the target concept as it is consistent with all
	examples under the probability distribution $P$. Thus, simply outputting a
	negative example drawn from $D$ simulates the oracle $\EX(P, c^*)$.
	Likewise, the oracle $\EX(Q)$ is simulated by outputting the input part of a
	positively labeled example drawn from $D$. Let $h : X \rightarrow \{0, 1,
	\question \}$ be the hypothesis returned by $L(\epsilon/2, \delta, \EX(P, c^*),
	\EX(Q))$. We use $h$ to define a classifier $g : X \rightarrow \{0, 1\}$ as
	follows. Let $g(x) = 0$ if $h(x) = 0$, and $g(x) = 1$ if $h(x) \in \{\question,
	1\}$. Then we have the following: 
	\begin{align*}
		\false^+_D(g; D) &= \P{(x, y) \sim D}{g(x) = 1 \wedge y = 0}  \\
							&\leq \P{(x, y) \sim D}{g(x) = 1 \mid y = 0} \\
							&= \P{(x, y) \sim D}{h(x) = \question \mid y = 0} + \P{(x, y) \sim D}{h(x) = 1 \mid y = 0} \\
							&\leq \rej_P(h) + \err_P(h; c^*) \leq \eps/2 + \eps/2 = \eps.
	\end{align*}
    In the above we have used the rejection and $\err_P$ bound discussed above. Likewise, 
	\begin{align*}
		\false^-_D(g) &= \P{(x, y) \sim D}{g(x) = 0 \wedge y = 1} \\
							&= \P{(x, y) \sim D}{h(x) = 0 \wedge y = 1} \\
							&\leq \P{(x, y) \sim D}{c^*(x) = 0  \wedge y = 1} + \P{(x, y) \sim D}{h(x) \neq c^*(x) \wedge y = 1} \\
							&\leq \P{(x, y) \sim D}{c^*(x) = 0  \wedge y = 1} + \P{(x, y) \sim D}{h(x) \neq c^*(x) \mid y = 1} \\
							&= \opt_+ + \err_Q(h; c^*) \leq \opt_+ + \epsilon.
	\end{align*}
\end{proof}

\section{Separation Results}
\label{sec:learning_results}
\subsection{Algorithm for Learning Parities}
In this section, we observe that a very simple algorithm can PQ-learn the class
$\PAR$ of parity functions over $\{0, 1\}^d$.  For $T \subseteq \{1, 2, \ldots,
d \}$, the parity function $\oplus_T(x) := \oplus_{i \in T} x_i$ is $1$ if an
odd number of bits in $S$ are $1$. Note that the parity function corresponding
to the empty set $\emptyset$ is the constant $0$ function. If we also include
the constant function $1$, which is not expressible as a parity, in $\PAR$, the
reductions from the previous section also give positive and negative reliable learners
for the class $\PAR$.
\footnote{The inclusion of the constant $1$ function is only required for \emph{negative reliable learning}. Without it there is no guarantee that there is any concept in the class that has no false negative errors.}
We consider learnability of the family of classes $\langle \PAR_d \rangle_{d \geq 1}$, where $\PAR_d$ represents the parity functions over $\zo^d$. Although the definitions in Section~\ref{sec:prelims} omit this issue for readability, it is common in learning theory; for efficient learnability we require that the running time is bounded by a polynomial in $d$, $1/\epsilon$ and $1/\delta$. 

\begin{lemma} \label{lem:parity}
	The class $\PAR$ is PQ-learnable. 
\end{lemma}
The proof idea is that we may reject all examples not in the span of the training data---all other examples have a label that can be uniquely determined from the training data. The complete proof appears in Appendix~\ref{app:learning_results_proofs}.

\subsection{Hardness of Learning Conjunctions}
The equivalence of PQ and reliable learning makes it easy to show that PQ learning is likely harder than PAC learning. In this section, we observe that positively reliably learning the class $C$ of conjunctions is as hard as the problem of PAC learning $\DNF$ formulae, a problem that has remained open since the seminal paper of~\citet{valiant1984theory} that introduced it. Learning $\DNF$ formulae is known to be at least as hard as learning juntas on $\log d$ variables, another notoriously hard problem~\citep{Blum93relevantexamples}, and has recently shown to be hard conditional on the hardness of refuting random $k$-SAT formulae~\citep{daniely2016complexity}. As in the previous section, we consider a family of concept classes $C_d$ parametrized by size complexity $d$. 

For $X =\zo^d$ and $S \subseteq \{1,2,\ldots d\}$, the conjunction $\wedge_S$
is the function which is 1 if all the bits of $S$ are 1. The family
$\wedge_{\pm d}$ of \textit{general conjunctions} includes conjunctions that
may have literals or negations of literals. We first note that if one can
positively reliably learn $\wedge_d$ then one can positively reliably learn
$\wedge_{\pm d}$. This follows from a standard representation trick---one
simply maps each example in $\zo^d$ to an example $x'\in \zo^{2d}$ by taking
each $x \in \zo^d$ and concatenating $x$ with the bits of $x$ negated. In this
representation, any general conjunction over $x$ corresponds to a conjunction
over $x'$. Although not directly relevant, we note that without noise, the
family $\wedge_d$ of conjunctions on $d$ variables is trivially learnable from
positive examples alone (or online with at most $d$ mistakes), and thus also
negative reliably learnable. 

The class of $s$-term $\DNF$ formulae over $\{0, 1\}^d$ consists of boolean
functions that can be represented as a disjunction of at most $s$ terms, $
\varphi = T_1 \vee T_2 \vee \cdots T_s$, where each term $T_i \in \wedge_{\pm
d}$. It is easy to see that a positive reliable learner for conjunctions can be
used to \emph{weakly} learn $s$-term $\DNF$ formulae. If the labels under the
distribution are not (almost) balanced, then either the constant $0$ or $1$ function is
already a weak learner. Otherwise, each of the conjunctions $T_i$ classifies all
examples labeled negatively by $\varphi$ as negative, and at least one of the
$T_i$'s classifies $\frac{1}{s}$ fraction of the examples labeled positively by
$\varphi$ correctly. Thus, a positive reliable learner for conjunctions yields
a $\Omega(1/s)$ weak learner for $s$-term $\DNF$. A standard boosting algorithm can be
then used to convert this to obtain a PAC learning algorithm. We
remark that observations along these lines have already been made in previous
work (e.g.~\cite{Bshouty2005MaximizingAW, KKM:2012}). The above discussion can
be formalized in the form of the following lemma whose proof is omitted.

\begin{lemma}\label{lem:juntas}
If conjunctions are positively reliably learnable, then polynomial-size $\DNF$ formulae are PAC learnable. 
\end{lemma}

\subsection*{Acknowledgments}

Varun Kanade was supported in part by the Alan Turing Institute under the EPSRC
grant EP/N510129/1.

\bibliography{src/refs}

\begin{thebibliography}{22}
\providecommand{\natexlab}[1]{#1}
\providecommand{\url}[1]{\texttt{#1}}
\expandafter\ifx\csname urlstyle\endcsname\relax
  \providecommand{\doi}[1]{doi: #1}\else
  \providecommand{\doi}{doi: \begingroup \urlstyle{rm}\Url}\fi

\bibitem[Ben-David et~al.(2010)Ben-David, Lu, Luu, and P{\'a}l]{ImpossibleCvS}
S.~Ben-David, Tyler Lu, Teresa Luu, and D.~P{\'a}l.
\newblock Impossibility theorems for domain adaptation.
\newblock In \emph{AISTATS}, 2010.

\bibitem[Ben-David and Urner(2012)]{ben2012hardness}
Shai Ben-David and Ruth Urner.
\newblock On the hardness of domain adaptation and the utility of unlabeled
  target samples.
\newblock In \emph{International Conference on Algorithmic Learning Theory},
  pages 139--153. Springer, 2012.

\bibitem[Ben-David et~al.(2008)Ben-David, Lu, and P{\'a}l]{ben2008does}
Shai Ben-David, Tyler Lu, and D{\'a}vid P{\'a}l.
\newblock Does unlabeled data provably help? {W}orst-case analysis of the
  sample complexity of semi-supervised learning.
\newblock In \emph{COLT}, pages 33--44, 2008.

\bibitem[Blum(1993)]{Blum93relevantexamples}
Avrim~L. Blum.
\newblock Relevant examples and relevant features: Thoughts from computational
  learning theory.
\newblock In \emph{In AAAI-94 Fall Symposium on 'Relevance', 1994.}, pages
  302--311, 1993.

\bibitem[Bshouty and Burroughs(2005)]{Bshouty2005MaximizingAW}
N.~Bshouty and L.~Burroughs.
\newblock Maximizing agreements with one-sided error with applications to
  heuristic learning.
\newblock \emph{Machine Learning}, 59:\penalty0 99--123, 2005.

\bibitem[Daniely and Shalev-Shwartz(2016)]{daniely2016complexity}
Amit Daniely and Shai Shalev-Shwartz.
\newblock Complexity theoretic limitations on learning {DNF}’s.
\newblock In \emph{Conference on Learning Theory}, pages 815--830, 2016.

\bibitem[Durgin and Juba(2019)]{durgin2019hardness}
Alexander Durgin and Brendan Juba.
\newblock Hardness of improper one-sided learning of conjunctions for all
  uniformly falsifiable {CSP}s.
\newblock In \emph{Algorithmic Learning Theory}, pages 369--382, 2019.

\bibitem[Goel et~al.(2017)Goel, Kanade, Klivans, and Thaler]{goel2017reliably}
Surbhi Goel, Varun Kanade, Adam Klivans, and Justin Thaler.
\newblock Reliably learning the {ReLU} in polynomial time.
\newblock In \emph{Conference on Learning Theory}, pages 1004--1042. PMLR,
  2017.

\bibitem[Goldwasser et~al.(2020)Goldwasser, Kalai, Kalai, and
  Montasser]{GKKM:2020}
Shafi Goldwasser, Adam~Tauman Kalai, Yael~Tauman Kalai, and Omar Montasser.
\newblock Beyond perturbations: Learning guarantees with arbitrary adversarial
  test examples.
\newblock \emph{Advances in Neural Information Processing Systems 33: Annual
  Conference on Neural Information Processing Systems 2020}, 2020.
\newblock URL \url{https://arxiv.org/abs/2007.05145}.

\bibitem[Hopkins et~al.(2019)Hopkins, Kane, and Lovett]{hopkins2019power}
Max Hopkins, Daniel~M Kane, and Shachar Lovett.
\newblock The power of comparisons for actively learning linear classifiers.
\newblock \emph{arXiv preprint arXiv:1907.03816}, 2019.

\bibitem[Huang et~al.(2007)Huang, Gretton, Borgwardt, Sch{\"o}lkopf, and
  Smola]{huang2007correcting}
Jiayuan Huang, Arthur Gretton, Karsten Borgwardt, Bernhard Sch{\"o}lkopf, and
  Alex~J Smola.
\newblock Correcting sample selection bias by unlabeled data.
\newblock In \emph{Advances in neural information processing systems}, pages
  601--608, 2007.

\bibitem[Juba(2016)]{juba2016learning}
Brendan Juba.
\newblock Learning abductive reasoning using random examples.
\newblock In \emph{Proceedings of the Thirtieth AAAI Conference on Artificial
  Intelligence}, pages 999--1007, 2016.

\bibitem[Kalai et~al.(2012)Kalai, Kanade, and Mansour]{KKM:2012}
Adam~Tauman Kalai, Varun Kanade, and Yishay Mansour.
\newblock Reliable agnostic learning.
\newblock \emph{Journal of Computer and System Sciences}, 78\penalty0
  (5):\penalty0 1481--1495, 2012.

\bibitem[Kanade and Thaler(2014)]{kanade2014distribution}
Varun Kanade and Justin Thaler.
\newblock Distribution-independent reliable learning.
\newblock In \emph{Conference on Learning Theory}, pages 3--24, 2014.

\bibitem[Kearns et~al.(1992)Kearns, Schapire, Sellie, and
  Hellerstein]{Kearns92towardefficient}
Michael~J. Kearns, Robert~E. Schapire, Linda~M. Sellie, and Lisa Hellerstein.
\newblock Toward efficient agnostic learning.
\newblock In \emph{In Proceedings of the Fifth Annual ACM Workshop on
  Computational Learning Theory}, pages 341--352, 1992.

\bibitem[Kivinen(1990)]{kivinen1990reliable}
Jyrki Kivinen.
\newblock Reliable and useful learning with uniform probability distributions.
\newblock In \emph{Proceedings of the First International Workshop on
  Algorithmic Learning Theory (ALT)}, pages 209--222, 1990.

\bibitem[Li et~al.(2011)Li, Littman, Walsh, and Strehl]{li2011knows}
Lihong Li, Michael~L Littman, Thomas~J Walsh, and Alexander~L Strehl.
\newblock Knows what it knows: a framework for self-aware learning.
\newblock \emph{Machine learning}, 82\penalty0 (3):\penalty0 399--443, 2011.

\bibitem[Pietrzak(2012)]{pietrzak2012cryptography}
Krzysztof Pietrzak.
\newblock Cryptography from learning parity with noise.
\newblock In \emph{International Conference on Current Trends in Theory and
  Practice of Computer Science}, pages 99--114. Springer, 2012.

\bibitem[Pitt and Valiant(1988)]{pitt88}
Leonard Pitt and Leslie~G. Valiant.
\newblock Computational limitations on learning from examples.
\newblock \emph{J. ACM}, 35\penalty0 (4):\penalty0 965–984, October 1988.
\newblock ISSN 0004-5411.
\newblock \doi{10.1145/48014.63140}.
\newblock URL \url{https://doi.org/10.1145/48014.63140}.

\bibitem[Rivest and Sloan(1988)]{rivest1988learning}
Ronald~L Rivest and Robert~H Sloan.
\newblock Learning complicated concepts reliably and usefully.
\newblock In \emph{AAAI}, pages 635--640, 1988.

\bibitem[Sayedi et~al.(2010)Sayedi, Zadimoghaddam, and Blum]{sayedi2010trading}
Amin Sayedi, Morteza Zadimoghaddam, and Avrim Blum.
\newblock Trading off mistakes and don't-know predictions.
\newblock In \emph{Advances in Neural Information Processing Systems}, pages
  2092--2100, 2010.

\bibitem[Valiant(1984)]{valiant1984theory}
Leslie~G Valiant.
\newblock A theory of the learnable.
\newblock \emph{Communications of the ACM}, 27\penalty0 (11):\penalty0
  1134--1142, 1984.

\end{thebibliography}
\ifarxiv
\bibliographystyle{plainnat}
\fi

\newpage
\appendix

\section{Proof from Section~\ref{sec:learning_results}}
\label{app:learning_results_proofs}
The proof of Lemma~\ref{lem:parity} gives an algorithm for PQ Learning Parities.
\begin{proof}[\ifarxiv Proof \fi of Lemma~\ref{lem:parity}]
	Let $f$ be the target parity distribution. Observe that the $\VC$ dimension of $\PAR$ is $d$ and let $m = \poly(d, 1/\epsilon, 1/\delta)$ be large enough that so that for a sample $S$ of size $m$ drawn from the oracle $\EX(P, f)$ satisfies with probability at least $1 - \delta$, that for all $c \in \PAR$, 
	\[ |\err_{\delta_S}(c; f) - \err_P(c; f)| \leq \epsilon/2, \]
	where $\delta_S$ is the empirical distribution over the sample $S$.

	Let $\hat{c} \in \PAR$ be any parity function for which
	$\err_{\delta_S}(\hat{c}; f) = 0$. Such a $\hat{c}$ must exist as $f \in \PAR$,
	and in fact can be found efficiently using Gaussian elimination. Let $S = \{(x_1,
		y_1), \ldots (x_m, y_m)\}$ denote the training dataset. Considering $\{0,
	1\}^d$ as a vector space over $\GF(2)$, let $k$ denote the dimension of the
	vector space $V := \vspan\{x_1, \ldots, x_m \}$. If $k = d$, then $\hat{c}$
	is uniquely determined and is equal to the target parity $f$. Otherwise, let
	$\{a_1, \ldots, a_{d-k} \} \subseteq \{0, 1\}^d$ be a set of linearly
	independent vectors so that $\vspan\{x_1, \ldots, x_m, a_1, \ldots, a_{d -k}
	\} = \{0, 1\}^d$. Assigning a label of either $0$ or $1$ to the ``\emph{datum}''
	$a_i$ uniquely defines a parity that is consistent with $\hat{c}$ on $S$. Thus,
	there are $2^{d -k}$ possible choices of $\hat{c}$ that are consistent
	with $S$. Let us denote this set by $\PAR_S$. We use this fact to prove that
	$\P{x \sim P}{x \not\in V} \leq \epsilon$. Clearly, if $V$ is
	$d$-dimensional, then this probability is $0$. Otherwise any $x \not\in V$
	can be written as $x^\prime + \tilde{x} + a_i$ for some $i \in \{1, \ldots,
	d - k \}$ and for $x^\prime \in V$ and $\tilde{x} \in \vspan\{a_1, \ldots,
	a_{i-1}, a_{i+1}, \ldots, a_{d - k} \}$. Now for a random parity $\tilde{c}$
	drawn from $\PAR_S$, the probability that $f(x) \neq \tilde{c}(x)$ is
	exactly $\frac{1}{2}$. To see this, observe that a random parity can be
	chosen by first assigning a random label from $\{0, 1\}$ to each $a_j$, $j
	\neq i$. Then the label of $\tilde{c}(x)$ can still be either $0$ or $1$ and
	is completely determined by the random choice made for $a_i$. This means
	that the label assigned will be different from $f(x)$ with probability
	exactly $\frac{1}{2}$. Thus, we have: 
	\begin{align*} 
		\E{\tilde{c} \sim_U \PAR_S}{\err_P(\tilde{c}; f)} \geq \frac{1}{2} \P{x \sim P}{x \not\in V}
	\end{align*} 
	On the other hand, by our choice of $m$, with probability at least $1 - \delta$,
	we have that for each $\tilde{c} \in \PAR_S$, $\err_P(\tilde{c}; f) \leq
	\epsilon/2$. Hence, it must be the case that $\P{x \sim P}{x \not\in V} \leq
	\epsilon$. 

	Let $h : \zo^d \rightarrow \{0, 1, \bot\}$ be defined as follows: if $x \in V$, then $h(x) = \hat{c}(x)$, else $h(x) = \bot$. Clearly, for any $x \in V$, $\hat{c}(x) = f(x)$, hence $\err_Q(h; f) = 0$. On the other hand, $\rej_P(h) \leq \P{x \sim P}{x \not \in V} \leq \epsilon$. 
\end{proof}

\end{document}